\documentclass{article}

\usepackage{amsmath,amsfonts,bm}

\usepackage{thmtools,thm-restate}

\def\eqref#1{equation~(\ref{#1})}

\def\1{\bm{1}}

\DeclareMathAlphabet{\mathsfit}{\encodingdefault}{\sfdefault}{m}{sl}
\SetMathAlphabet{\mathsfit}{bold}{\encodingdefault}{\sfdefault}{bx}{n}

\newcommand{\E}{\mathbb{E}}

\DeclareMathOperator*{\argmax}{arg\,max}

\usepackage{natbib}
\usepackage{tikz}
\usetikzlibrary{positioning}
\usepackage{subcaption}
\usepackage{hyperref}
\usepackage{url}
\usepackage{graphicx}
\usepackage{xcolor}
\usepackage{enumerate}
\usepackage{enumitem}
\usepackage{textcomp}
\usepackage{mathrsfs}
\usepackage{amsthm}
\usepackage{amsmath}
\usepackage{amssymb}
\usepackage{mathtools}
\usepackage{bbm}
\usepackage{epigraph}
\usepackage{etoolbox}
\usepackage{algorithm}
\usepackage{algpseudocode}
\usepackage{thmtools, thm-restate}
\usepackage{proof-at-the-end}
\usepackage{authblk}
 
\newtheorem{theorem}{Theorem}

\newtheorem{definition}[theorem]{Definition}

\def\nn{\nonumber}

\def\E{\mathbb{E}}

\def\Pr{\mathbb{P}}

\def\1{\mathbf{1}}

\def\hat{\widehat}

\usepackage{xcolor}
\usepackage{latexcolors}
\usepackage{hyperref}
\definecolor{stanfordred}{rgb}{0.54901961, 0.08235294, 0.08235294}
\usepackage[mathscr]{euscript}
\hypersetup{
    colorlinks=true,
    linkcolor=blue,
    filecolor=magenta,      
    urlcolor=purple,
    citecolor = stanfordred,
}

\newcount\Comments
\newcommand{\kibitz}[2]{\ifnum\Comments=1{\textcolor{#1}{\textsf{\footnotesize #2}}}\fi}
\usepackage{color}
\definecolor{darkred}{rgb}{0.7,0,0}
\definecolor{darkgreen}{rgb}{0.0,0.5,0.0}
\definecolor{darkblue}{rgb}{0.0,0.0,0.5}
\definecolor{teal}{rgb}{0.0,0.5,0.5}

\newcommand{\theirs}{agglomerative}

\title{Inclusive Artificial Intelligence}
\author[]{Dilip Arumugam\thanks{\texttt{dilip@cs.stanford.edu}} }
\author[]{Shi Dong\thanks{\texttt{sdong15@stanford.edu}} }
\author[]{Benjamin Van Roy\thanks{\texttt{bvr@stanford.edu}}}
\affil[]{Stanford University}
\date{\today}

\begin{document}

\maketitle

\begin{abstract}
Prevailing methods for assessing and comparing {\it generative AI}s incentivize responses that serve a hypothetical representative individual.  Evaluating models in these terms presumes homogeneous preferences across the population and engenders selection of {\it agglomerative AIs}, which fail to represent the diverse range of interests across individuals.  We propose an alternative evaluation method that instead prioritizes {\it inclusive AI}s, which provably retain the requisite knowledge not only for subsequent response customization to particular segments of the population but also for utility-maximizing decisions.
\end{abstract}

\section{Introduction}

{\it Generative AIs} have the potential to deliver tremendous value to society at large. The current leading approach is characterized by enormous models, ranging from billions to trillions of parameters, that first undergo a pretraining phase facilitated by huge swaths of data from the World Wide Web before honing their competency through interaction with humans.  Given a natural language prompt, for example, such an AI can generate a useful response, whether it be text, artwork, or computer code.  While much discourse has focused on the extent to which generative AIs truly understand language~\citep{bender-koller-2020-climbing,manning2022human}, they are transcending the standard language model to serve a wide variety of needs~\citep{ziegler2019fine,bommasani2021opportunities,bai2022training,huggingface2022rlhf} including text summarization~\citep{stiennon2020learning,wu2021recursively}, web navigation~\citep{nakano2021webgpt}, open-ended text generation~\citep{ouyang2022training}, question answering~\citep{menick2022teaching}, and dialogue~\citep{glaese2022improving,openai2022chatgpt}. 

Curiously, prevailing methods for assessing and comparing generative AIs incentivize the eventual selection of {\it agglomerative AIs}, which produce responses that serve a single, prototypical individual meant to represent the collective predilections of the entire population; clearly, such a societal archetype is not only hypothetical but also entirely fictitious given the diverse range of interests and perspectives across the population. The ramifications of such monomaniac selection procedures can prove to be severe, rendering AIs incapable of subsequent alignment to the preferences of individuals or sub-populations.

By {\it inclusive AIs}, we refer to those that aim to represent heterogeneity across the population through their responses.  Where an agglomerative AI strives for an optimal response to each prompt, an inclusive AI produces a distribution of responses that reflect diverse preferences.  Incidentally, current pretraining practices do produce a sort of inclusive AI, as models come to mimic variations across responses observed in Web data.  On the other hand, current fine-tuning\footnote{This work, in particular, exclusively focuses on fine-tuning with reinforcement learning and human feedback~\citep{stiennon2020learning,ouyang2022training}, rather than fine-tuning as it appears in the supervised-learning literature.} practices, used to subsequently learn from human interaction, tend to agglomerate and are rewarded for this by methods that are used to assess and compare competing models.  Not only does this fail to refine knowledge about diversity of preferences but it corrupts such knowledge accumulated during pretraining.

We will propose an alternative assessment method that prioritizes selection of {\it inclusive AI}s.  The aim is to avoid compromising holistic societal preferences in favor of any dominating majority.  Instead, AIs that fare well according to our assessment should sharpen the latent corpus of knowledge acquired during pre-training to best reflect the overall population. With rapid progress in this area underway and community norms around best practices already beginning to converge, careful consideration and restructuring of evaluation methods will be paramount to ensuring that generative AIs remain beneficent and afford value to society at large.

\section{Comparing Generative AIs}
\label{sec:comparing-generative-ais}

We take a generative AI to be a function mapping each prompt to a distribution over possible responses; naturally, a single response can then be sampled from this distribution.  As is common in AI benchmarking, we consider a head-to-head contest in which two AIs are each presented with a sequence $(X_t:t=1\ldots,T)$ of prompts.  We will refer to the competing AIs as $A$ and $B$ -- the comparison can be thought of as an $A/B$ test~\citep{johari2017peeking}.  For each $t \in \{1,2,\ldots,T\}$, the first model produces a response distribution $P_A(\cdot \mid X_t)$ and sampled response $Y_{t,A} \sim P_A(\cdot \mid X_t)$; analogously, the second produces distribution $P_B(\cdot \mid X_t)$ and response $Y_{t,B} \sim P_B(\cdot \mid X_t)$.  An annotator is asked to indicate their preference $L_t(A,B) \in \{0,1\}$, which is $1$ if $A$ is preferred and $0$ if $B$.  For simplicity, we will assume that each annotator is sampled randomly from the human population.

\subsection{Agglomerative Objective}

A first impulse may lead us to score each AI according to the number of times it is preferred:
\begin{align}
\label{eq:agglomerative-objective}
S^{\mathrm{agg}}(A,B) = \sum_{t=1}^T L_t(A,B).
\end{align}
Then, we could choose an AI by comparing $S^{\mathrm{agg}}(A,B)$ against $S^{\mathrm{agg}}(B,A)$.
We will refer to this scoring function as the {\it agglomerative objective}.
Indeed, this is the essence of approaches used to compare state-of-the-art generative AIs~\citep{ziegler2019fine,ouyang2022training,bai2022training,glaese2022improving,bakker2022fine}.  When the competitor's identity is clear from context, we will suppress the arguments and simply write $S^{\mathrm{agg}}_A$ instead of $S^{\mathrm{agg}}(A,B)$.

A limitation of the agglomerative objective is that it encourages agglomerative AIs.  In particular, as we will establish under mild conditions in Theorem \ref{th:deterministic-agglomerative-ai}, there exists an AI that only ever outputs a Dirac delta distribution (that is, a one-hot probability mass function) and fares at least as well as any other with respect to the agglomerative objective. In other words, the agglomerative objective is optimized by an agglomerative AI, which outputs a unique response to each prompt.

\subsection{Inclusive Objective}

We propose as an alternative scoring function:
\begin{align}
\label{eq:inclusive-objective}
S^{\mathrm{inc}}(A,B) =  \sum_{t=1}^T &\Big\{L_t(A,B) \log P_A(Y_{t,A} \mid X_t) + (1-L_t(A,B)) \log P_A(Y_{t,B} \mid X_t) \nn\\
&- \log \Big(P_A(Y_{t,A} \mid X_t) + P_A(Y_{t,B}  \mid  X_t)\Big)\Big\}.
\end{align}
We will refer to this as the {\it inclusive objective}.  Again, when the competitor's identity is clear from context, we will suppress the arguments and simply write $S^{\mathrm{inc}}_A$ instead of $S^{\mathrm{inc}}(A,B)$.

A didactic example serves to illustrate how these two objectives select different AIs.  Consider competing AIs $A$ and $B$ compared across $T=3$ trials, each with the same prompt: $X_1 = X_2 = X_3$.  Suppose there are two possible responses $\{1,2\}$, the first preferred by two-thirds of the population and the second by one-third.  AI $A$ is inclusive, assigning probabilities $P_A(1 \mid X_t) = 2/3$ and $P_A(2 \mid X_t)=1/3$ to the two responses.  AI $B$ is agglomerative: assigning all probability to the first response: $P_B(1 \mid X_t) = 1$ and $P_B(2 \mid X_t) = 0$.  Consider responses generated by the two AIs as illustrated in Figure \ref{fig:objectives_example}.  Choices made by annotators are shaded in green.  The agglomerative objective selects AI $B$ because it is preferred in two of the three trials.  The inclusive objective, on the other hand, evaluates to $2 \log (2/3) + \log(1/3)$ for AI $A$ and $-\infty$ for AI $B$, and thus selects AI $A$.  The reason AI $B$ fares so poorly is that, on the third trial, it assigns probability $P_{B}(1 \mid X_3) = 1$ to the response $1$, which conveys extreme confidence, though it turns out that the annotator chooses response $0$.  It is also worth noting that, in the first two trials, the inclusive objective does not penalize AI $A$ though its response is not chosen.  This is because, instead of the choice, the objective judges the probability $P_A(1 \mid X_3) = P_A(1 \mid X_3) = 2/3$ that the AI assigned to what was chosen.

\begin{figure}[htb]
\begin{center}
\includegraphics[scale=0.6]{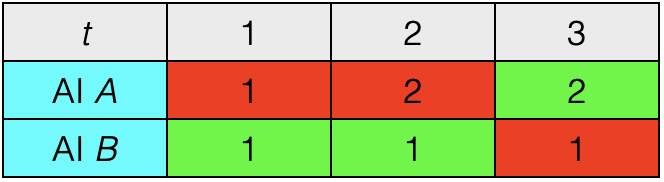}
\end{center}
\caption{A didactic example of responses generated by two AIs over three trials.  Annotator choices for the preferred response are shaded in green.}
\label{fig:objectives_example}
\end{figure}

\subsection{Motivating Examples}

The key merit of our inclusive objective is precisely in its ability to avoid the selection of agglomerative AIs.
An inclusive AI retains information about diverse preferences across the population, ultimately resolving to a distribution that, in many applications, can better serve downstream users than a single, maximally-preferable response. To help clarify the potential benefits, consider a representative use case: an amateur writer who is not only looking to produce topically-relevant prose but also has certain stylistic preferences around diction. An idealized workflow for this individual would likely consist of prompting the system with some text, sampling a few dozen articles written in response, and choosing their favorite. Note, however, that this selection step becomes far less useful to the writer when the response distribution places all probability mass on a single response.


A devil's advocate might suggest that there actually {\it ought} to be a single optimal response and that what this example overlooks is the possibility of a more elaborate prompt that either fully reflects the writer's preferences or encourages the generation of multiple candidate options.  In the former case, the prompt could present everything relevant to selecting among articles, potentially obviating the need for a distribution; however, the cognitive load imposed by such detailed specification makes this impractical.  Indeed, much of the value offered by generative AIs stems from the {\it reduction} of cognitive effort required in creative processes. To that same effect, one might entertain the idea of a single, best-response model that itself produces an entire response distribution as its singular response. Yet, clearly, such a distribution with an overwhelmingly large (albeit, finite) support is far too cumbersome for a recipient to make sensible determinations about how to best utilize it. Thus, once again, cognitive load for the end user stands as an impediment to practical use of such an agglomerative AI. 

Alternatively, one might maintain that a response distribution is merely an intermediate output and, ultimately, what users care about are representative i.i.d. samples from this distribution, which an agglomerative AI could easily be prompted to generate. While this style of prompting is perfectly plausible, recall that the agglomerative objective for model selection entails a head-to-head contest where, in this context, each prompt essentially specifies some underlying response distribution and human evaluators must determine which AI produces a more plausible sequence of i.i.d. samples from this distribution; results from the cognitive-science literature already inform us that humans are demonstrably bad at this exact task of perceiving and judging randomness~\citep{griffiths2001reconciling,williams2013people}, even for observations generated from the simplest of distributions (sequences of observed coin flips, as an example). Consequently, model selection based on such human evaluations is likely to be fraught with errors and inaccuracies; in contrast, our inclusive objective preferentially gives rise to AIs which yield such response distributions without requiring human judgement to assess the accuracy of the underlying probabilities.

While this work largely focuses on the role and benefits of response randomization when selecting between generative AIs, there is an important nuance between the kinds of uncertainty under consideration. In particular, the preceding example has focused on the benefits of representing {\it aleatoric} uncertainty~\citep{der2009aleatory} such that, even if the AI has full knowledge of preferences across the entire population, a prompt which does not fully convey the individual's preferences entails a desirable response that is itself a random variable from the AIs perspective.  An additional benefit supported by inclusive AIs arises with representation of {\it epistemic} uncertainty that manifests due to the insufficiency of training data. Consider a generative AI designed to aid in the diagnosis of medical ailments where each prompt describes a patient's symptoms and each response offers a diagnosis. Suppose a particular prompt presents highly unusual symptoms, unlike any observed during pretraining.  Rather than an agglomerative AI that only produces a single diagnosis, there can be great value to an inclusive AI that expresses its lack of confidence and knowledge (that is, its high degree of epistemic uncertainty) through a response distribution.  This uncertainty can play
a critical role in a physician's treatment while also highlighting a potential, productive path where such generative AIs can augment, rather than replace, physicians and their skill sets in order to accelerate the diagnostic process~\citep{gottesman2019guidelines}.

\section{Theoretical Analysis}
\label{sec:theory}

We next establish theoretical results that formalize the differences between objectives and the benefits of inclusive AIs. In the interest of space, all proofs of results stated in this section are relegated to the appendix.

\subsection{Formalism}



We model uncertain quantities as random variables, each defined with respect to a common probability space $(\Omega, \mathcal{F}, \mathbb{P})$.
Let $\mathcal{X}$ be the set of prompts and $\mathcal{Y}$ the set of responses.  
To simplify analysis, we assume that both $\mathcal{X}$ and $\mathcal{Y}$ are finite.
We assume that the sequence $(X_t: t=1,\dots, T)$ of prompts is i.i.d.  Recall that a generative AI $A$, maps each prompt $x \in \mathcal{X}$ to a probability distribution $P_A(\cdot \mid x)$ over responses.

To model how preferences vary across the population, we consider {\it random utilities}.
After generative AIs $A$ and $B$ produce responses $Y_{t,A}$ and $Y_{t,B}$ respectively to a prompt $X_t$, the prompt-response pairs are presented to an annotator drawn uniformly at random from the population.  The random annotator assigns a real-valued utilities $U_t(X_t, Y_{t,A})$ and $U_t(X_t, Y_{t,B})$ to the responses and provides binary feedback, with $L_t = 1$ if $U_t(X_t, Y_{t,A}) > U_t(X_t, Y_{t,B})$ and $L_t = 0$ if $U_t(X_t, Y_{t,A}) < U_t(X_t, Y_{t,B})$.  If $U_t(X_t, Y_{t,A}) = U_t(X_t, Y_{t,B})$, $L_t$ is based on a fair coin toss.

For a fixed prompt, different individuals can prefer different responses.  For each prompt $X \in \mathcal{X}$, we denote by $\overline{P}(\cdot \mid X)$ the distribution of the response that would be selected by a uniformly sampled member of the population who can choose any response in $\mathcal{Y}$.  We will refer to $\overline{P}$ the {\it population choice distribution}.
To understand the relation between this object and the sampled individual's utility function $\overline{U}$, note that if the individual selects a response $Y \in \mathcal{Y}$ then $\Pr(Y \mid X) = \overline{P}(Y \mid X)$ and $\overline{U}(X, Y) = \max_{y \in \mathcal{Y}} \overline{U}(X, y)$.

\subsection{Objective Function Outcomes}





Our inclusive objective (\ref{eq:inclusive-objective}) is meant to engender AIs that represent variations in preferences across the population.
At the extreme, we could hope for what we will refer to as a {\it maximally inclusive} AI, which fully reflects the population preference distribution.

\begin{definition}
An AI $A$ is {\bf maximally inclusive} with respect to a population choice distribution $\overline{P}$ if, for all $x\in\mathcal{X}$ and $y\in\mathcal{Y}$, $P_A(y \mid x) = \overline{P}(y \mid x)$.
\end{definition}
The above definition implies that for a given population choice distribution, a maximally inclusive AI exists and is unique. In what follows, we will fix a population choice distribution $\overline{P}$, and denote the corresponding maximally inclusive AI as $A^\star$.
Our next theorem shows that, if we use the \theirs{} objective, AIs that produce deterministic responses for each prompt are always favored.

\begin{restatable}{theorem}{deterministicagglomerativeai}
\label{th:deterministic-agglomerative-ai}
For any generative AI $A$, there exists a generative AI $B$, such that $P_B(y \mid x) \in \{0,1\}$ for all $(x,y)\in\mathcal{X}\times\mathcal{Y}$, and 
$$
\mathbb{P}( S_B^{\rm agg} \geq S_A^{\rm agg}) \geq 1/2
\quad\text{and}\quad
\E[S_B^{\mathrm{agg}}] \geq \E[S_A^{\mathrm{agg}}].
$$
\end{restatable}

Next, we will demonstrate that, under mild assumptions, the maximally inclusive AI scores at least as well as any other AI in terms of the inclusive objective.
For this purpose, we first introduce a particular random utility model, the {\it multinomial logit}, which is a very general model of discrete choice \citep{mcfadden2000mixed} used in classic literature on ranking through pairwise comparison~\citep{bradley1952rank}.
The multinomial logit model posits a population utility function $\overline{u}:\mathcal{X}\times\mathcal{Y}\to \mathbb{R}$.
For any generative AI $A$, conditioned on the prompt-response pair $(x, y)$, the random utility $U_{t}(x, y)$ is such that
\[
    U_{t}(x, y) = \overline{u}(x, y) + \xi_{t,x,y},
\]
where each $\xi_{t,x,y}$ is an independent standard Gumbel random variable.
Hence, if two AIs $A$ and $B$ face off in a head-to-head-contest,
\begin{align*}
\mathbb{P}\Big(L_t(A,B) = 1\bigm\vert X_t, Y_{t,A}, Y_{t,B}\Big) &= \mathbb{P}\Big(U_t(X_t, Y_{t,A}) \geq U_t(X_t, Y_{t,B}) \bigm\vert X_t, Y_{t,A}, Y_{t,B}\Big)\\ 
&= \frac{e^{\overline{u}(X_t, Y_{t,A})}}{e^{\overline{u}(X_t, Y_{t,A})} + e^{\overline{u}(X_t, Y_{t,B})}}.
\end{align*}

The following result establishes that, if preferences across the population are consistent with a multinomial logit model then, in terms of our inclusive objective, a maximally inclusive AI will fare at least as well as any other in expectation.

\begin{restatable}{theorem}{inclusiveoptimal}
\label{th:inclusive-optimal}
Under the multinomial logit model and for any generative AI $B$, $\E[S^{\mathrm{inc}}(A^\star, B)] \geq \E[S^{\mathrm{inc}}(B, A^\star)]$.
\end{restatable}

\subsection{Downstream Decision Making}

In this section, we consider a downstream decision problem and demonstrate how inclusive AIs can support more effective decisions than agglomerative ones.  For simplicity and clarity of exposition, we consider a setting with a fixed prompt $x\in\mathcal{X}$ such that $X_1=X_2=\dots=X_T=x$, and a finite set of actions $\mathcal{A}$. 
The decision-maker aims to choose an action $a\in\mathcal{A}$ that maximizes $V(a) = \mathbb{E}\big[v(a, Y)\big]$, where $Y$ represents the favorite response of a random individual drawn uniformly from the population and $v:\mathcal{A}\times \mathcal{Y}\mapsto [0,1]$ is a value function.  
Instead of querying humans, the decision-maker repeatedly presents an AI $A$ with the prompt $x$, generating $T$ responses $(Y_{t, A}:t=1,\dots,T)$.  The decision-maker then selects an action $\hat{a}_T$ that maximizes the empirical mean:
\[
    \hat{a}_T \in \argmax_{a\in\mathcal{A}} \left\{\frac{1}{T}\cdot\sum_{t=1}^T v(a, Y_{t, A})\right\}.
\]
Our next two results demonstrate that inclusive AIs lead to near-optimal decisions whereas agglomerative AIs can fail miserably.  We quantify the degree of inclusivity in terms of the KL divergence $D_{\rm KL}(Y_{t,A^\star}(x) \| Y_{t,A}(x))$ between the response distributions of an arbitrary AI $A$ and a maximally inclusive $A^\star$.

\begin{restatable}{theorem}{decisionone}
\label{thm:decision}
For any AI $A$ and $\delta>0$,
\[
V(\hat{a}_T) \geq \max_{a\in\mathcal{A}} V(a) - 2\sqrt{D_{\rm KL}(Y_{t,A^\star}(x)\|Y_{t,A}(x))} - \sqrt{\frac{2}{T}\log\left( \frac{2|\mathcal{A}|}{\delta}\right)},
\]
with probability at least $1-\delta$.
\end{restatable}

\begin{restatable}{theorem}{decisiontwo}
\label{thm:decision2}
There exists an AI $A$ and value function $v$ such that 
\[\E[S^{\mathrm{agg}}(A,A^\star)] \geq \E[S^{\mathrm{agg}}(A^\star,A)] \quad \text{and} \quad
V(\hat{a}_T) \leq \max_{a\in\mathcal{A}} V(a) - \frac{1}{3}.
\]
\end{restatable}

Intuitively, Theorem \ref{thm:decision} highlights how inclusive AIs support near-optimal decision making, as the second term of the lower bound is zero for a maximally inclusive AI (by virtue of capturing the population choice distribution) whereas the final term vanishes as the number of responses $T$ grows. For an arbitrary generative AI that is not necessarily the maximally inclusive AI, the result shows a graceful degradation in the optimality of downstream decisions as the AI response distribution deviates from the population choice distribution. Meanwhile, Theorem \ref{thm:decision2} demonstrates that, at least for some examples, the preferred model under the agglomerative objective necessarily yields sub-optimal, unimprovable downstream decisions, regardless of the number $T$ of responses.

\section{Closing Remarks}

In this work, we have outlined shortcomings of the model selection procedure driving recent developments in generative AI. In particular, we recognize that the underlying criterion encourages the finetuning and subsequent selection of agglomerative AIs that presume homogeneity in the preferences of the overall population and deterministically produce a single best response for each input prompt. To remedy this, we have proposed an alternative criterion and, through simple examples alongside a corroborating theoretical analysis, demonstrated that the resulting inclusive AIs retain the requisite heterogeneity in responses needed for downstream specialization to individuals or sub-populations. We leave to future work the open question of how well this theory translates into practice and yields successful empirical instantiations of inclusive AIs. In the remainder of this section, we take a step back to situate our work within the broader context of model alignment in AI.

While our work broadly falls in with a burgeoning literature around ensuring the alignment of AIs with societal values, research on this rich topic is extremely varied, reflecting the considerable number of potential issues that can arise from the pervasive use of these generative AIs~\citep{bender2021dangers,weidinger2021ethical,tamkin2021understanding,weidinger2022taxonomy,ganguli2022red,ngo2022alignment}. We refer readers to the excellent preceding surveys for a comprehensive treatment of the topic at large, as well as more targeted works spanning complex, critical issues including bias against marginalized communities~\citep{birhane2022power}, dissemination of misinformation~\citep{lin-etal-2022-truthfulqa}, and consensus on controversial topics (for example, moral or political issues)~\citep{hendrycks2021aligning,kasirzadeh2022conversation,bakker2022fine}. Taking an orthogonal line to these works, our focus in this paper is on the methodology surrounding the evaluation of such generative AIs and the objective functions used to determine which model yields the most preferable responses~\citep{askell2021general,bommasani2022picking,abramson2022improving}. In the context of Reinforcement Learning from Human Feedback (RLHF)~\citep{ouyang2022training,huggingface2022rlhf}, this design choice ultimately governs the underlying reward function optimized when reinforcement learning is applied to subsequently finetune the pretrained models and, accordingly, has profound impact on whether the resulting models produce outputs that align with our expectations. 

While the standard regime for model selection from human preferences~\citep{christiano2017deep} operates under the classic Bradley-Terry model~\citep{bradley1952rank} of assuming that preferences in the selected options of human evaluators reflect their underlying utility preferences, \citet{abramson2022improving} consider scenarios where highly disparate experiences may impair an evaluator's ability to clearly discern utility and, as a result, yield inaccurate preference judgements; they propose a modification to the Bradley-Terry model that presents evaluators with comparisons from within the same temporal interaction. \citet{askell2021general} take a step back and compare the preference modeling methodology on the whole against imitation learning and binary classification approaches to eliciting judgements for model alignment; overall, they find preference model to be the more suitable technique relative to standard imitation (supervised) learning. In contrast to these prior works, our contributions are targeted towards elucidating the collapse and homogenization of response distributions produced by generative AIs under the current, widely-adopted model selection technique for RLHF, something noted but not rectified in recent work~\citep{bakker2022fine,bommasani2022picking}. 

Our work most directly aligns with the focus of \citet{bommasani2022picking} who, while not operating in the RLHF setting, study homogeneity resulting from particular instances of generative AIs and, specifically, highlight the precise choice of finetuning mechanism as a strong influencer of this homogeneity. In keeping with their findings, our object of study is the objective function that underlies model selection in the RLHF paradigm which, in turn, exerts strong impact on the subsequent finetuning via reinforcement learning. We outline a more precise form of homogeneity in the collapse of the response distribution to the population mean, rather than including the overall population distribution. Most recently, \citet{bai2022constitutional} highlight a simple but effective technique for refining the RLHF finetuning stage in order to mitigate harmful responses while still maintaining helpful, salient outputs for input prompts; much like their work,  we hope that the ideas presented here afford a viable path towards rectifying some of the  weaknesses of RLHF~\citep{perez2022discovering}, expanding model capabilities for everyone without exacerbating model biases or excluding certain demographics within the population.

\bibliographystyle{plainnat}
\bibliography{references}

\appendix

\section{Proofs}

\deterministicagglomerativeai*
\begin{proof}
Let $x\in\mathcal{X}$ be a fixed prompt.  
For each generative AI $A$, letting $Y_A(x)$ be the random response generated by $A$, we first show that we can always find a deterministic $y_0(x)\in\mathcal{Y}$, such that for all $t=1,2,\dots,T$,
\begin{equation}
    \label{eq:domination}
    \mathbb{P}\Big( U_t\big(x, y_0(x)\big) \geq U_t\big(x, Y_A(x)\big)\Big) \geq\frac{1}{2},
\end{equation}
In fact, suppose that
\begin{equation*}
    \mathbb{P}\Big( U_t\big(x, y(x)\big) \geq U_t\big(x, Y_A(x)\big)\Big) < \frac{1}{2}
\end{equation*}
for some $t$ and all $y\in\mathcal{Y}$, then we can consider a $\mathcal{Y}$-valued random variable $Y'$ that is i.i.d.~as $Y_A(x)$.  Apparently $\mathbb{P}\big( U(x,Y') \geq U(x, Y_A(x)) \big) = 1/2$, but at the same time,
\[
    \mathbb{P}\big( U_t(x,Y') \geq U_t(x, Y_A(x)) \big) = \sum_{y\in\mathcal{Y}} \Big[\mathbb{P}\big( U_t(x,y) \geq U_t(x, Y_A(x)) \big)\cdot \mathbb{P}\big( Y' = y) \Big]< \frac{1}{2},
\]
leading to a contradiction.

Note that \eqref{eq:domination} implies that
\[
    \mathbb{P}\Big( U_t\big(x, y_0(x)\big) > U_t\big(x, Y_A(x)\big)\Big) \geq\mathbb{P}\Big( U_t\big(x, y_0(x)\big) < U_t\big(x, Y_A(x)\big)\Big).
\]
We can consider an AI $B$ such that
\[
    P_{B}(y \mid x) = 
    \begin{cases}
        1 & \text{if $y=y_0(x)$}\\
        0 & \text{otherwise}
    \end{cases},\quad \forall x\in\mathcal{X}.
\]
Apparently $P_{B}(y\mid x)\in\{0,1\}$ for all $x$ and $y$.  Meanwhile, for each $t=1,2,\dots$,
\begin{eqnarray}
    \mathbb{P}\Big(L_t(B, A) = 1 \Big)
    &=& \mathbb{P}\Big(U_t(X_t, y_0(X_t)) > U_t(X_t, Y_A(X_t)) \Big)\nn\\
        &&+ \frac{1}{2} \mathbb{P}\Big(U_t(X_t, y_0(X_t)) = U_t(X_t, Y_A(X_t)) \Big)\nn\\
    &\geq& \mathbb{P}\Big(U_t(X_t, y_0(X_t)) < U_t(X_t, Y_A(X_t)) \Big)\nn\\
        &&+ \frac{1}{2} \mathbb{P}\Big(U_t(X_t, y_0(X_t)) = U_t(X_t, Y_A(X_t)) \Big)\nn\\
    &=& \mathbb{P}\Big(L_t(B, A)=0 \Big).\nn 
\end{eqnarray}
Hence $\mathbb{P}(L_t(B, A) = 1 )\geq 1/2$ for each $t$, leading to
\[
    \mathbb{P}\Big( S^{\rm agg}_{B} \geq S^{\rm agg}_{A}\Big) \geq \frac{1}{2},
\]
as we desire.  Since $L_t(B, A) + L_t(A, B) = 1$, this also means that
\[
    \E\big[S^{\rm agg}_{B} \big] \geq \E\big[S^{\rm agg}_{A} \big].
\]
\end{proof}

\inclusiveoptimal*
\begin{proof}
Let $\overline{u}$ be the population utility function in the multinomial logit model.
Such population utility function $\overline{u}$ gives rise to a population choice distribution $\overline{P}_{\overline{u}}$.
To see that, we can consider a fixed prompt $x$ and a random individual that is asked to pick out the best response among all possible responses $y\in\mathcal{Y}$.  According to the multinomial logit model, the probability that $y_0$ is chosen at $t$ is given by
\[
    \mathbb{P}\Big( \overline{u}(x, y_0) + \xi_{t, x, y_0} = \max_{y\in\mathcal{Y}} \big\{ \overline{u}(x,y) + \xi_{t,x,y}\big\}\Big)
    = \frac{e^{\overline{u}(x,y_0)}}{\sum_{y\in\mathcal{Y}} e^{\overline{u}(x,y)}}.
\]
Since the individual is drawn uniformly at random from the distribution, we can conclude that the fraction of population that considers $y_0$ the best response to $x$ is 
\[
    \overline{P}_{\overline{u}}(y_0\mid x) = \frac{e^{\overline{u}(x,y_0)}}{\sum_{y\in\mathcal{Y}} e^{\overline{u}(x,y)}}.
\]

According to the definition of maximal inclusiveness, for all prompt-response pair $(x,y)\in\mathcal{X}\times\mathcal{Y}$, 
\[
    P_{A^\star}(y \mid x) = \overline{P}_{\overline{u}}(y \mid x).
\]
Let $B$ be an arbitrary generative AI.
At timestep $t$, let $Y_{t, A^\star}$ and $Y_{t, B}$ be the responses generated by $A^\star$ and $B$, respectively.   Also let
\[
    Q_{t, \star} = \frac{P_{A^\star}(Y_{t,A^\star} \mid X_t)}{P_{A^\star}(Y_{t,A^\star} \mid X_t) + P_{A^\star}(Y_{t,B} \mid X_t)}
\]
and
\[
    Q_{t, B} = \frac{P_B(Y_{t,A^\star} \mid X_t)}{P_B(Y_{t,A^\star} \mid X_t) + P_B(Y_{t,B} \mid X_t)}
\]
be the relative probabilities that AIs $A^\star$ and $B$ assign to $Y_{t, A^\star}$, respectively.
Using $L_t$, $S^{\rm inc}_{A^\star}$, and $S^{\rm inc}_{B}$ as shorthands for $L_t(A^\star, B)$, $S^{\rm inc}(A^\star, B)$, and $S^{\rm inc}(B, A^\star)$, respectively, we arrive at
\begin{eqnarray}
    \label{eq:tower-property}
    &&\mathbb{E} \Big[ S^{\rm inc}_{A^\star} - S^{\rm inc}_{B}\Big]\nn\\
    &=&  \sum_{t=1}^T\mathbb{E} \left[ L_t\cdot \log \frac{Q_{t, \star}}{Q_{t, B}} + (1-L_t) \cdot\log \frac{1-Q_{t,\star}}{1-Q_{t,B}}\right]\nn\\
    &=& \sum_{t=1}^T\mathbb{E}\left[\mathbb{E} \left[ L_t\cdot \log \frac{Q_{t, \star}}{Q_{t, B}} + (1-L_t) \cdot\log \frac{1-Q_{t,\star}}{1-Q_{t,B}}\Bigg| X_t, Y_{t, B}, Y_{t, A^\star}\right]\right].
\end{eqnarray}
According to the multinomial logit model, conditioned on $X_t$, $Y_{t,B}$, and $Y_{t, A^\star}$, $\mathbb{P}(L_t=1) = Q_{t, \star}$.  Letting $Z_t = L_t - Q_{t, \star}$, we have that $$\E\Big[Z_t\big|X_t, Y_{t, B}, Y_{t, A^\star}\Big] = 0.$$
Using the fact that $Q_{t, B}$ and $Q_{t, \star}$ are measurable with respect to the $\sigma$-algebra generated by $X_t$, $Y_{t,B}$, and $Y_{t, A^\star}$,
\begin{eqnarray}
    \label{eq:kl}
    &&\mathbb{E} \left[ L_t\cdot \log \frac{Q_{t, \star}}{Q_{t, B}} + (1-L_t) \cdot\log \frac{1-Q_{t,\star}}{1-Q_{t,B}}\Bigg| X_t, Y_{t, B}, Y_{t, A^\star}\right]\nn\\
    &=& \mathbb{E}\left[Z_t\cdot \left(\log \frac{Q_{t, \star}}{Q_{t, B}} + \log \frac{1-Q_{t,\star}}{1-Q_{t,B}}\right)\Bigg| X_t, Y_{t, B}, Y_{t, A^\star} \right] + D_{\rm KL}\big( \bm{Q}_{t, \star} \big\| \bm{Q}_{t, B} \big)\nn\\
    &=& D_{\rm KL}\big( \bm{Q}_{t, \star} \big\| \bm{Q}_{t, B} \big)\ \geq\ 0,
\end{eqnarray}
where we use $\bm{Q}_{t, \star}$ and $\bm{Q}_{t, B}$ to denote the distributions $(Q_{t, \star}, 1-Q_{t, \star})$ and $(Q_{t, B}, 1-Q_{t, B})$, respectively.  Plugging \eqref{eq:kl} into \eqref{eq:tower-property}, we arrive at
\[
    \mathbb{E} \Big[ S^{\rm inc}_{A^\star} - S^{\rm inc}_{B}\Big] \geq 0.
\]
\end{proof}

\decisionone*
\begin{proof}
    To avoid cluttering, we will fix a prompt $x\in\mathcal{X}$ in this proof, and write $Y_{t,A}(x)$ and $Y_{t,A^\star}(x)$ simply as $Y_{t,A}$ and $Y_{t,A^\star}$, respectively.
    Let $\epsilon = D_{\rm KL}\big(Y_{t,A^\star}\|Y_{t,A}\big)$.
    For each $a\in\mathcal{A}$, let $\hat{V}_T(a)=\frac{1}{T}\sum_{t=1}^T v(a, Y_{t,A})$.
    For each $a\in\mathcal{A}$, the function $v(a, \cdot)$ is a mapping from finite set $\mathcal{Y}$ to $[0,1]$.
    As a result, by the data-processing inequality, 
        \[
            D_{\rm KL}\Big(v(a, Y_{t,A^\star})\big\|v(a, Y_{t,A}) \Big) \leq D_{\rm KL}\big(Y_{t,A^\star}\|Y_{t,A}\big) = \epsilon.
        \]
        From Pinsker's inequality, letting $D_{\rm TV}(\cdot, \cdot)$ be the total variation distance, we have
        \[
            D_{\rm TV}\Big(v(a, Y_{t,A^\star}), v(a, Y_{t,A}) \Big) \leq \sqrt{\frac{1}{2}D_{\rm KL}\Big(v(a, Y_{t,A^\star})\big\|v(a, Y_{t,A}) \Big)} \leq \sqrt{\frac{\epsilon}{2}}.
        \]
        Since both $v(a, Y_{t,A^\star})$ and $v(a, Y_{t,A})$ are bounded in $[0,1]$, this leads to
        \begin{equation}
            \label{eq:expectation-difference}
            \Big|\mathbb{E}\big[v(a, Y_{t,A})\big]- V(a)\Big|\leq\sqrt{\frac{\epsilon}{2}},\quad\forall a\in\mathcal{A},
        \end{equation}
        where we used the fact that
        \[
            \mathbb{E}\big[v(a, Y_{t,A^\star})\big] = \mathbb{E}\big[v(a, Y)\big] = V(a).
        \]
        Meanwhile, following from the Azuma-Hoeffding inequality, for all $\zeta>0$,
        \[
            \mathbb{P}\left(\left| \hat{V}_T(a) - \mathbb{E}\big[v(a, Y_{t,A})\big] \right| \geq \zeta\right) \leq 2\exp\left( -2T\zeta^2\right),
        \]
        meaning that, with probability at least $1 - 2|\mathcal{A}|\exp(-2T\zeta^2)$,
        \begin{equation}
            \label{eq:union-bound}
            \left| \hat{V}_T(a) - \mathbb{E}\big[v(a, Y_{t,A})\big] \right| < \zeta,\quad\forall a\in\mathcal{A}.
        \end{equation}
        Together with \eqref{eq:expectation-difference}, the above event implies that
        \[
            \left| \hat{V}_T(a) - V(a) \right| < \sqrt{\frac{\epsilon}{2}}+\zeta,\quad\forall a\in\mathcal{A}.
        \]
        Letting $a^\star\in\argmax_{a\in\mathcal{A}} V(a)$, under the event represented by \eqref{eq:union-bound}, we have that
        \begin{eqnarray}
            &&\max_{a\in\mathcal{A}} V(a) - V\big( \hat{a}_T \big)\nn\\
            &=& \Big[V(a^\star) - \hat{V}_T(a^\star) \Big] + \Big[\hat{V}_T(a^\star) - \hat{V}_T\big( \hat{a}_T \big) \Big] + \Big[\hat{V}_T\big( \hat{a}_T \big) - V\big( \hat{a}_T \big)\Big]\nn\\
            &\leq& \Big[V(a^\star) - \hat{V}_T(a^\star) \Big]  + \Big[\hat{V}_T\big( \hat{a}_T \big) - V\big( \hat{a}_T \big)\Big]\nn\\
            &<& \sqrt{2\epsilon} + 2\zeta,
        \end{eqnarray}
        where the first inequality results from that
        \[
            \hat{a}_T \in \argmax_{a\in\mathcal{A}} \hat{V}_T(a).
        \]
        Letting $\delta = 2|\mathcal{A}|\exp(-2T\zeta^2)$, we arrive at our desired result.
\end{proof}

\decisiontwo*
\begin{proof}
Fix a prompt $x \in \mathcal{X}$. We can consider a case where there are only two responses and two actions, i.e. $\mathcal{Y} = \{y_1, y_2\}$, $\mathcal{A} = \{a_1, a_2\}$.  The value function $v$ is given by
    \begin{eqnarray}
        v(a_1, y_1) = 1-\beta, && v(a_1, y_2) = 1,\nn\\
        v(a_2, y_1) = 1, && v(a_2, y_2) = 0,\nn
    \end{eqnarray}
    for some $\beta > 0$.
    Let the population choice distribution be given as 
    \[
        \overline{P}(y_1|x) = \overline{P}(y_2|x) = \frac{1}{2},
    \]
    leading to
    \[
        P_{A^\star}(y_1|x) = P_{A^\star}(y_2|x) = \frac{1}{2}.
    \]
    We can see that the expected values of the two actions are
    \[
        V(a_1) = 1 - \frac{\beta}{2}, \quad V(a_2) = \frac{1}{2}.
    \]
    Consider an AI $A$ that deterministically produces $y_1$ as response.  Apparently 
    \[
        \E\big[ S^{\rm agg}(A, A^\star)\big] \geq \E\big[ S^{\rm agg}(A^\star, A)\big].
    \]
    Meanwhile, since $v(a_1,y_1) < v(a_2, y_1)$, the decision-maker tends to choose action $a_2$ if they only observe the response generated by $A$.  As such, for any $T\geq 1$, $\hat{a}_T = a_2$.  This results in
    \[
        \max_{a\in\mathcal{A}} V(a) - V\big( \hat{a}_T \big) = \frac{1-\beta}{2}.
    \]
    Letting $\beta = 1/3$, we arrive at the result.
\end{proof}

\end{document}